\documentclass[lettersize,journal]{IEEEtran}
\usepackage{amsmath,amsfonts}
\usepackage{algorithm}
\usepackage{algorithmicx}
\usepackage{array}
\usepackage[caption=false,font=normalsize,labelfont=sf,textfont=sf]{subfig}
\usepackage{textcomp}
\usepackage{stfloats}
\usepackage{url}
\usepackage{verbatim}
\usepackage{graphicx}
\def\BibTeX{{\rm B\kern-.05em{\sc i\kern-.025em b}\kern-.08em
    T\kern-.1667em\lower.7ex\hbox{E}\kern-.125emX}}
\usepackage{balance}
\usepackage[utf8]{inputenc}
\usepackage{amsmath}
\usepackage{amsfonts}
\usepackage{amssymb}
\usepackage{amsthm}
\usepackage{color}
\usepackage{algorithm}
\usepackage{algpseudocode}
\usepackage{bbm}
\usepackage{bm}
\usepackage{graphicx}
\usepackage{lipsum}
\usepackage{bbold}

\algdef{SE}[DOWHILE]{Do}{doWhile}{\algorithmicdo}[1]{\algorithmicwhile\ #1}%

\newtheorem{theorem}{Theorem}
\newtheorem{lemma}{Lemma}
\newtheorem{assumption}{Assumption}

\newcommand\Tstrut{\rule{0pt}{2.6ex}}

\newcommand{\Yd}{\mathbf{Y}^{(d)}}
\newcommand{\yd}{\mathbf{y}^{(d)}}
\newcommand{\ydd}{\mathbf{y}'^{(d)}}
\newcommand{\Ydd}{\mathbf{Y}'^{(d)}}
\newcommand{\Y}{\mathbf{Y}}
\newcommand{\y}{\mathbf{y}}
\newcommand{\Ydash}{\mathbf{Y}'}
\newcommand{\X}{\mathbf{X}}
\newcommand{\Xdash}{\mathbf{X}'}

\newcommand{\Hhat}{\hat{H}}
\newcommand{\Pmax}{|P|_{max}}
\newcommand{\Pmaxx}{(1+d(\Pmax+1))}
\newcommand{\phat}{\hat{p}}
\newcommand{\prob}{\mathbf{P}}
\newcommand{\E}{\mathbb{E}}
\newcommand{\w}{\mathbf{w}}
\newcommand{\z}{\mathbf{z}}
\newcommand{\tm}{\mathcal{P}}
\newcommand{\eo}{\mathbf{e}_{0}}
\newcommand{\Mbar}{\bar{M}}

\newcommand{\init}{\Yd(0)=\yd_0, \Ydd(0) = \ydd_0}

\newcommand{\0}{\mathbb{\tau_0}}

\begin{document}
\title{Learning the Influence Graph of a High-Dimensional Markov Process with Memory}
\author{Smita Bagewadi and Avhishek Chatterjee
\thanks{Authors are with the Department of Electrical Engineering, Indian Institute of Technology Madras, Chennai, India.}
\thanks{AC thanks the Department of Science and Technology, India, for support through the grant INSPIRE/04/2016/001171.}}

\maketitle

\begin{abstract}
Motivated by multiple applications in social networks, nervous systems, and financial risk analysis, we consider the problem of learning the underlying (directed) influence graph or causal graph of a high-dimensional multivariate discrete-time Markov process with memory. At any discrete time instant, each observed variable of the multivariate process is a binary string of random length, which is parameterized by an unobservable or hidden $[0,1]$-valued scalar. The hidden scalars corresponding to the variables evolve according to discrete-time linear stochastic dynamics dictated by the underlying influence graph whose nodes are the variables. We extend an existing algorithm for learning i.i.d. graphical models to this Markovian setting with memory and prove that it can learn the influence graph based on the binary observations using logarithmic (in number of variables or nodes) samples when the degree of the influence graph is bounded. The crucial analytical contribution of this work is the derivation of the sample complexity result by upper and lower bounding the rate of convergence of the observed Markov process with memory to its stationary distribution in terms of the parameters of the influence graph.
\end{abstract}


\section{Introduction}
\label{sec:intro}
\IEEEPARstart{L}{earning} the underlying influence graph, a.k.a. causal graph, of a high dimensional multivariate Markov process is an intriguing mathematical problem with many applications \cite{Chatterjee,Ravazzi1,Quinn,Hall}. It is central to many disciplines, including but not limited to social network analysis, neural and biological signal processing, and financial risk analysis. In this paper, we study the problem of learning the influence graph of the following high dimensional $d$-Markov process, i.e., a Markov process with $d$ memory, from its samples.

 There is an underlying directed graph $G$ with nodes $V$ and directed edges $E$. At a given time $t$, any node $v \in V$ has an internal parameter $X_v(t) \in [0,1]$ and generates multiple (possibly random number of) independent instances of Bernoulli$(X_v(t))$ random variables. These Bernoulli variables at $v$ over the last $d$ time steps, in turn, impact the future $X_u(t+1)$ of node $u$  through weighted averaging if there is a directed edge $(v,u) \in E$ from $v$ to $u$. Our problem is to learn the graph $G$ by observing only the instances of the Bernoulli variables across all nodes. 

Below, we discuss some application areas where the above problem is often encountered.

{\bf Social networks:} In many online social platforms, agents express their opinions in binary form as likes or dislikes. The expressed opinions on the platform influence the internal or private beliefs of other agents. In particular, an agent's internal belief is a weighted average of the recently expressed opinions of the agents it regularly follows \cite{De},\cite{Ravazzi1},\cite{Ravazzi2},\cite{degroot1},\cite{Das}. In turn, the internal or private belief of an agent dictates the nature of its expressed opinions. Thus, the process of expressed opinions of all agents in a social network is indeed a high-dimensional $d$-Markov process of the above kind.  It is known that only a few agents on the friend list of an agent on a social platform have a real influence on it, and often, an agent not on the friend list has an influence either through the physical world or through other platforms. Hence, one cannot directly use the friend list as the influence graph; rather, one has to learn the true influencers using the high-dimensional process of expressed binary opinions.

{\bf Brain and biological networks:} In neuroscience, there has been significant interest in finding the hidden logical connections between different centers of the brain and nervous systems \cite{Lazaro}, \cite{Quinn}, \cite{Zhang}. The only observable data are the EEG or other measurements of different brain and nerve centers. Neuronal triggers or firings have inherent casuality and Markovianity \cite{Soltani}. The firing rate of a nerve center, i.e., the number of (binary) spikes in a time interval of appropriate duration, reflects the internal excitement of that center. % time averaged 
These spikes then contribute to the excitement of other nervous centers causally and cumulatively \cite{Soltani}, \cite{Fourcaud}. Thus, the overall process of neuronal triggers across many centers behaves like a Markov 
process with binary observables. Similar dynamics are also observed in other biological networks, for example, gene regulatory networks \cite{Cowen}.

 {\bf Financial risk analysis: } Major financial changes in a company are often influenced by other companies and worldwide financial and administrative bodies, some of which may not even have any direct financial link with the company. Thus, for accurate financial risk analysis, learning such hidden connections is extremely important \cite{Bernardete},\cite{Yiqi}. In practice, such connections have to be learned based on the raw data of the performance of different companies and the actions of different bodies in the past. This, in essence, is a problem of learning the causal graph of a high-dimensional Markov process. Though the above-mentioned binary model is not a perfect representation of this scenario, it is a tractable first-order approximation.

\subsection{Main Contribution}
\label{sec:contribution}
Under the condition that the in-degree of any node is $O(1)$, we present a conditional directed entropy-based polynomial-time greedy algorithm for learning the causal influence graph based on the binary observations. For provable recovery of the true causal graph with high probability, we provide an upper bound on sample complexity that scales logarithmically with the number of nodes. The greedy algorithm presented here is an extension of a greedy algorithm for learning i.i.d. graphical models \cite{Ray} to our Markovian dynamics.

 Our main mathematical contribution lies in the derivation of a sample complexity bound. An important component of our proof technique builds on a relatively recent and tighter result regarding a Hoeffding-style Markov concentration \cite{cineq} using the second eigenvalue. The main novelty of our proof technique lies in establishing an upper bound on the second eigenvalue in terms of the parameters of the Markov dynamics. We achieve this by upper bounding the total variation distance from the stationary distribution in terms of the parameters of the dynamics and comparing that with a lower bound obtained in terms of the second eigenvalue.

\subsection{Organization}
\label{sec:organization}
The rest of the paper is organized as follows. In Sec.~\ref{sec:literature}, we discuss existing literature that shares some connection with our work. We present the $d$-Markov model in detail in Sec.~\ref{sec:model}. In Sec.~\ref{sec:UsefulMathematicalStructures}, we discuss some useful mathematical structures of the model. Next, we present the algorithm and its performance guarantees in Sec.~\ref{sec:Algorithm} followed by the detailed proofs in Sec.~\ref{sec:proofs}. In Sec.~\ref{sec:NumericlResults}, we present some simulation results that corroborate and complement our analytical bounds. Finally, we conclude in Sec.~\ref{sec:conclusion}.

\section{Related Literature}
\label{sec:literature}
To the best of our knowledge, no analytical work has been reported in the literature on the particular mathematical problem considered here. However, our problem has partial but interesting overlaps with multiple areas of theoretical machine learning and applied mathematics. From the point of view of applications, as discussed in Sec.~\ref{sec:intro}, this work is related to learning influence in social networks and learning logical connections in nervous systems. Below, we discuss some of the related literature, starting with application-oriented works.

\subsection{Data Driven Applications}

 In \cite{De}, the authors learn the strength of the influence in a social network graph on each edge, assuming continuous time linear opinion evolution and the knowledge of the structure of the influence graph, given only the expressed opinions. An agent's opinion at a time is measured as a real number by using natural language processing based sentiment analysis of the posted texts. 
A discrete-time continuous-valued opinion evolution model was considered in \cite{Ravazzi2}, \cite{Ravazzi1}, and the underlying influence strengths were learned by posing well-behaved optimization problems. In \cite{Ravazzi2}, the authors exploit the linear structure to learn the influences using only opinions at time $t=0$ and $\infty$ across multiple sample paths. On the other hand, influences are learned in \cite{Ravazzi1} using one sample path by exploiting the underlying sparsity of the influence graph.

Over the last decade or so, there has been interest in understanding causal influence patterns between different parts of the brain or nervous system while performing a specific cerebral or motor task \cite{Lazaro,Quinn,Zhang,Soltani}. The problem of understanding brain communities using temporal information flow in the brain was studied in \cite{Lazaro}. Zhang et al. worked on estimating the functional connectivity in a brain using a sparse hidden Markov model \cite{Zhang}. Quinn et al. used estimates of directed information to understand causal relationships between different neural spike trains from different brain centers \cite{Quinn}. 

The similarity between genes having the same phenotype was learned using random walks on a network of interacting genes \cite{Cowen}. On the other hand, learning the underlying causal influence graph for financial networks has been undertaken in \cite{Bernardete,Yiqi}.

%%%%%%%%%%%%%%%%%%%%%%%%%%%%%%%%%%%%%%%%%%%%%%%
\subsection{Analytical Guarantees}

Learning the underlying graphical model of a Markov Random Field (MRF) has been a flourishing area of research in the last decade. Extensive research has been done on learning graphical models using i.i.d. samples from the MRF, for example, \cite{Ravikumar,Animashree,Netrapalli,Scarlett,Dasarathy,Hamilton}. One of the major challenges in learning graphical models was the locally diamond-like structure \cite{Netrapalli} and the assumption on correlation decay. These were settled in \cite{Ray}, which proved efficient learning in the presence of a diamond-like structure and without invoking correlation decay. However, the non-degeneracy assumption for conditional entropy was necessary. Further, in \cite{Bresler}, algorithms were presented for learning antiferromagnetic Ising models, where there is no correlation decay.

Learning causality in a multi-variate process has a rich history \cite{Dahlhaus,Bach,Jalali,Chatterjee,Hosseini}. Directed mutual information was used to learn causality between two stationary and ergodic processes in \cite{Quinn}. Learning the underlying parameters of two causally related Markov processes have been studied in \cite{pmc}. 

In \cite{bar}, authors provide sample complexity guarantees for learning the influence graph of a special Markov process, the Bernoulli autoregressive (BAR) process. The binary random variables associated with the nodes are Bernoulli with probability derived from the past states of their neighbors. Though \cite{bar} has quite a few similarities with our work, there are some major differences as well: (i) in our problem, each node generates multiple  (possibly random number) of Bernoulli variables at each time, (ii) we do not assume the knowledge of the degree of each node and a lower bound on the stationary probabilities and  (iii) we use directed conditional entropy based greedy algorithms.

The work closest to the current paper is \cite{Chatterjee}, where learning the causal graph of a general discrete-time Markov chain is considered. However, in \cite{Chatterjee}, many assumptions had to be imposed on the Markov process to obtain sample complexity guarantees. These assumptions include a lower bound on the smallest value of the stationary distribution, the absence of a diamond-like directed structure, and correlation decay-like assumption for Markov processes. By using ideas for learning i.i.d. graphical model from \cite{Ray}, some recent results in concentration inequalities \cite{cineq} and an approach developed by us for bounding the second eigenvalue of the particular Markov chain in question, we could relax those assumptions. Also, unlike the bound in \cite{Chatterjee}, which is in terms of the mixing time, our bound is directly in terms of the edge weights and the graph structure.

\section{Model}
\label{sec:model}

The pattern of influence is represented by a directed graph $G=(V,E)$. The relationship between the nodes $u, v \in V$ is represented by the edge $(u,v) \in E$. The edge $(u,v) \in E$ implies that node $u$ influences node $v$. We consider directed edges for capturing asymmetry in the influence, i.e., for $u \neq v$, $(u,v) \neq (v,u)$. The edge $(v,v)$ captures the self-influence of $v$ and is assumed to be present for each node. The neighborhood  $\mathcal{N}_v$ of node $v$ is defined as the set of nodes that influence $v$, i.e.,  $\mathcal{N}_v:=\{u \neq v: (u,v) \in E\}$.

We consider a setting where a binary observed variable is evolving with time over the above network. In particular, we consider a discrete-time system, where at time $t$, node $v$ generates observation `$1$' with probability $X_v(t)$, where $X_v(t)$ is an internal parameter associated with the node.

Based on the internal parameter value at time $t$, each node generates $M_v(t)$ instances of the observed variable. We model $M_v(t)$ as a random variable with distribution $\min(\text{Poisson}(\mu_v(X_v(t))),\Mbar)+1$, where $\Mbar$ is a constant and for each $v$, $\mu_v(\cdot):[0,1] \to \mathbb{R}$ is a $L$-Lipschitz function. Among these $M_v(t)$ binary random variables, $N_v(t)$ are positive (i.e `1'), where $N_v(t) \sim \text{Binomial}(M_v(t),X_v(t))$.

 Corresponding to each edge $(u,v)$ there are $d$ non-negative weights $a^{(r)}_{uv}$, $r=0, 1, \ldots d-1$. These weights capture the strength of the influence of the observed value of $u$ at $t-r$ on the internal parameter of $v$ at $t$. For each $v$, the weights satisfy the condition $\sum_{r=0}^{d-1} \sum_u a^{(r)}_{uv}=1$.

The evolution of the internal parameter of any node $v$ is affected by the observed variables of the nodes in $\mathcal{N}_v$, its own inner bias $l_v$, and the random fluctuations $Z_v(t) \in [0,1]$ with mean $0<\bar{z}_v<1$. The internal parameters evolve according to the following update equation. For any $v \in V$,
\begin{align}
\label{eq:1_model}
  X_v(t+1) &= (1-\alpha_v)((1-\beta)Z_v(t+1) + \beta l_v) \nonumber \\
  &+ \alpha_v\sum_{u \in \mathcal{N}_v \cup \{v\} }\sum_{r=0}^{(d-1)}a_{uv}^{(r)}\left(\frac{N_u(t)}{M_u(t)}\right)
\end{align}
where $\alpha_v,\beta \in (0,1)$.

Intuitively, the parameter $X_v(t)$ represents some internal process that is affected by the generated instances of the observed variable. It evolves as a weighted average of the past instances of the observed variable of its neighbors, its own inner bias, and some unpredictable fluctuations. The parameter $a_{uv}^{(r)}$ represents the weight given by node $v$ at time $t$ to the observed variables of node $u$ at time $t-r$. The parameter $\alpha_v$ captures the openness of node $v$, i.e., the weight it gives to the status of other nodes against its own inner bias and random fluctuations.

\bigskip
\noindent \textbf{The learning problem:}
{ The central problem in this paper is to learn the underlying influence graph based on the observed variables over a finite time window. Mathematically, given} $\{(N_v(t),{M}_v(t)): v \in V\}_{t=0}^{T-1}$  learn the underlying influence graph $G=(V,E)$, { i.e., learn the neighborhood $\mathcal{N}_v$ for all nodes $v \in V$}.

\section{Useful Mathematical Structures}
\label{sec:UsefulMathematicalStructures}

Given the observation $\{(N_v(t), M_v(t)): v \in V\}_{t=0}^{T-1}$, we construct the processes $Y_v(t):=\frac{N_v(t)}{M_v(t)}$, $\Yd_v(t):=\left(Y_v(t), Y_v(t-1), \ldots, Y_v(t-d+1)\right)$, $\Yd_Q(t):=\{\Yd_v(t): v \in Q \subseteq V\}$, $\Yd_{u,Q}(t):=\{\Yd_v(t): v \in Q \cup \{u\}, Q \subseteq V\}$ and $\Yd(t):=\{\Yd_v(t): v \in V \}$. From the dynamics in \eqref{eq:1_model}, it can be seen that $\Yd(t)=\{\Yd_v(t): v \in V \}$ satisfies the following lemma.

\begin{lemma}
\label{lem:1_Markov}
The process $\Yd(t)$ is a Markov chain. Moreover, for any node $v$ and at any time $t$, given $\Yd_{\mathcal{N}_v \cup \{v\}}(t)$,
$\Yd_v(t+1)$ is independent of all other past process values $\{\Yd_{V\setminus \{\mathcal{N}_v \cup \{v\}\}}(\tau): \tau \le t\}$ and $\{\Yd_{\mathcal{N}_v \cup \{v\}}(\tau): \tau < t\}$.
\end{lemma}

When $\Yd(t)$ is stationary, for nodes $u,v \in V$,  $\mathcal{N}_v\subseteq Q$, and $u \notin Q$, the probability $\prob(y_{v_{+}}|\yd_{v,Q}) = \prob(y_{v_{+}}|\yd_{v,Q,u})$ for all $y_{v_{+}},\yd_{v,Q,u}$, where $\yd_{v,Q}$ represents the value of the process $\Yd_{v,Q}(t)$ at the current time instant and $y_{v_{+}}$ represents the value of the process $Y_v(t)$ at the next time instant.

On the other hand, if $\mathcal{N}_v \not\subseteq Q$, then $\prob(y_{v_+}|\yd_{v,Q}) \neq \prob(y_{v_+}|\yd_{v,Q,u})$ for some $y_{v_+},\yd_{v,Q,u}$. 
In fact, the following lower bound on the total variation distance between these two conditional distributions holds when $M_v(t)=1$ for all $v, t$

\begin{lemma}
\label{lem:TVepsilonLB}
If $d=1$, $M_v(t)=1$ for all $t$ and $v$, and $\mathcal{N}_v \setminus \{u\} \subset Q$ for some $Q \subset V$, then there exists $y_u \in \{0,1\}$ such that for all values of $\yd_{v,Q}$
$||\prob(y_{v_+}|\yd_{v,Q},y_u)-\prob(y_{v_+}|\yd_{v,Q})||_{1}$ is strictly bounded away from zero by a constant that does not depend on the number of nodes.
\end{lemma}

\begin{proof}[Proof of Lemma~\ref{lem:TVepsilonLB}]
We can write $\prob(Y_{v_+} = 1|\Yd_{v,Q,u})$ as $\E[X_{v_+}|\Yd_{v,Q,u}].$ By using eq~\ref{eq:1_model}, for $M_v(t)=1, d=1,$ the quantity $|\prob(Y_{v_+} = 1|\Yd_{v,Q,u})- \prob(Y_{v_+} = 1|\Yd_{v,Q})|$ equals 
\begin{align}
   \Big\lvert \alpha_v[a_{uv}(Y_u-\E[Y_u|Y_{v,Q}]) + \nonumber \\ \sum\limits_{w \in \mathcal{N}_v\setminus \{v,Q,u\}}\left(\E[Y_w|Y_{v,Q,u}] - \E[Y_w|Y_{v,Q}] \right)]\Big\rvert.
\end{align}
When $\mathcal{N}_v \subseteq Q \cup \{u\}$, i.e., when $u$ is the only neighbor not contained in $Q$, the set $\mathcal{N}_v\setminus \{v,Q,u\}$ is empty. Thus, in this case,
\begin{align}
\label{eq:Bernoulli_bound}
    & |\prob(Y_{v_+} = 1|\Y_{v,Q,u})- \prob(Y_{v_+} =
     1|Y_{v,Q})|
     \nonumber \\& = \alpha_v a_{uv}|(Y_u-\E[Y_u|Y_{v,Q}])|
\end{align}
$Y_u \in \{0,1\}$ and for any $y_{v,Q} \in \{0,1\}^{|Q|+1}, \E[Y_u|Y_{v,Q}=y_{v,Q}]) \in (0,1).$ If we consider the event $Y_u=0,$ we have 
\begin{align}
\label{eq:Bernoulli_bound}
    & |\prob(Y_{v_+} = 1|\Y_{v,Q},Y_u=0)- \prob(Y_{v_+} =
     1|Y_{v,Q})|
     \nonumber \\& = \alpha_v a_{uv}|\E[Y_u|Y_{v,Q}]|
\end{align}
$Y_u$ is a Bernoulli with probability $X_u$, $\E[Y_u|Y_{v,Q}=y_{v,Q}]) = \prob(Y_u = 1|Y_{v,Q}=y_{v,Q}) = \E[X_u|Y_{v,Q}=y_{v,Q}]$. The quantity $\E[X_u|Y_{v,Q}=y_{v,Q}]$ is lower bounded by $(1-\alpha_v)((1-\beta) Z_v(t) + \beta l_v).$ Which means that when $u$ is the only neighbor not included in $Q$, for all $\y_{v,Q}, \exists \text{ a } y_u$ such that $|\prob(Y_{v_+} = 1|\y_{v,Q}) - \prob(Y_{v_+}=1|\y_{v,Q},y_u=0)| > c_1 $ where $c_1$ is a constant strictly greater than $0$

\end{proof}

The above result also seems to be true for $M_v(t) > 1$ and $d>1$. However, the proof remains elusive for the general case due to certain technicalities. We make the following assumption for the general case.
\begin{assumption}
\label{assm:3_epsilon}
%***** same as Lemma 7*******
Consider any  $v \in V$, $u \in \mathcal{N}_v$, $Q \subset V$ and $u \not\in Q$. Then for any choice of $y_v$ and $y_Q$, there exists $y_u$ such that 
$\sum_{y_{v_+}} |\prob(y_{v_+}|\yd_{v,Q}) - \prob(y_{v_+}|\yd_{v,Q,u})| >  \sqrt{\epsilon'},$ where $\epsilon'>0$.

\end{assumption}

The algorithm presented here utilizes the conditional independence property presented in Lemma \ref{lem:1_Markov}  to determine the true neighborhoods of all the nodes. To check these independence conditions, the algorithm uses the directed conditional entropy of each node or variable of the process $\Yd(t)$.

 The directed conditional entropy of node $v$ given the set $Q$ is given by:  
$-\sum_{\Yd_{Q,v}(t)} \sum_{\Yd_v(t+1)} \mathbf{P}(\Yd_v(t+1),\Yd_{Q,v}(t)) \log \mathbf{P}(\Yd_v(t+1)|\Yd_{Q,v}(t))$, for some $Q \subset V$ and $v \not\in Q$. When the system is stationary,  this quantity would not depend on $t$, and we denote it as $H(v_+|v,Q)$. This can be computed by computing the joint entropy $H(v_+,v,Q)$ and $H(v,Q)$, where
$H(v_+,v,Q)=-\sum_{\Yd_{Q,v}(t),\Yd_v(t+1)} \mathbf{P}(\Yd_v(t+1),\Yd_{Q,v}(t)) \log \mathbf{P}(\Yd_v(t+1),\Yd_{Q,v}(t))$ and  $H(v,Q)=-\sum_{\Yd_{Q,v}(t)}  \mathbf{P}(\Yd_{Q,v}(t)) \log \mathbf{P}(\Yd_{Q,v}(t))$.

Based on Lemma~\ref{lem:1_Markov} and  Assumption~\ref{assm:3_epsilon} (extension of Lemma~\ref{lem:TVepsilonLB}), we obtain the following useful fact. 

\begin{lemma}
\label{lem:CondEntMarkov}
For any $v$ and $u \not\in \mathcal{N}_v$, if $Q \subset V$ such that $\mathcal{N}_v \subseteq Q$, then
$H(v_+|v,Q,u) = H(v_+|v,Q)$. Further, under Assumption \ref{assm:3_epsilon}, for any $Q \subset V$ and $u \in \mathcal{N}_v$ and $u \not\in Q$, $H(v_+|v, Q) - H(v_+|v,Q,u)  > \epsilon$, where $\epsilon >0$. %
\end{lemma}
Proof of this lemma is presented in Sec.~\ref{sec:proofs}. Intuitively, the lemma says that a node is a neighbor of $v$ only if adding that node to the conditioning set $Q$ decreases the directed conditional entropy $H(v_+|v,Q)$ of node $v$.
The algorithms presented here build on this insight to learn the neighborhood using directed conditional entropy. Note that when $Q$ does not contain $\mathcal{N}_v$, Lemma~\ref{lem:CondEntMarkov} does not guarantee that every node $u$ which exceeds the $\epsilon$ threshold is a true neighbor. The algorithm has to work around this fact so that no node is falsely added to the neighborhood.

In this problem, however, $H(v_+|v,Q)$ cannot be directly computed since the distributions are not known. Hence, they have to be estimated based on samples, which we denote by $\hat{H}(v_+|v,Q)$. We first obtain plug-in estimates the joint entropy $\hat{H}(v_+,v,Q)$ and $\hat{H}(v,Q)$ using the estimated joint distributions $\phat(\yd_v,\yd_Q)$ and $\phat(y_{v_+},\yd_v,\yd_Q)$ respectively. Then, we subtract them to obtain the conditional entropy estimates $\hat{H}(v_+|v,Q)$, which are used by the algorithms.

 Building on \cite{Ray}, we observe that in our case, if the error in the estimation of directed conditional entropy can be limited below a certain value, it is possible to select a threshold that can distinguish between neighbors and non-neighbors. This is stated mathematically in Lemma~\ref{lem:tau_modified}: 
\begin{lemma}
    \label{lem:tau_modified}
    Under Assumption \ref{assm:3_epsilon}, if $|\Hhat(v_+|v,P) - H(v_+|v,P)| \leq \frac{\epsilon}{4} \text{ for all }v \text{ and any  }P\subseteq V$ , then for all $u \in \mathcal{N}_v, v \in V, \text{ and } Q \subset V \setminus\{u,v\}$,
    \begin{equation}
    \label{thresh1}
       \Hhat(v_+|v,Q) - \Hhat(v_+|v,Q,u) > \frac{\epsilon}{2}
    \end{equation}
    
    and for all $w \notin \mathcal{N}_v, Q \subset V \setminus\{u,v,w\} \text{ such that } \mathcal{N}_v \subseteq Q$, 
    \begin{equation}
    \label{thresh2}
        \Hhat(v_+|v,Q) - \Hhat(v_+|v,Q,w) < \frac{\epsilon}{2}
    \end{equation}
        
\end{lemma}
\noindent Next, we define a few quantities that would be useful later. For $r=0,1,\hdots,d-1$, let $\Tilde{A}(r)$ be a matrix whose $(v,u)^{th}$ element $[\Tilde{A}(r)]_{vu}=\alpha_v a^{(r)}_{uv}$. Let $\rho(\Tilde{A}(r))$ be the largest eigenvalue of $\Tilde{A}(r)$ and let us define $\rho(\Tilde{A}):=\max \limits_r \rho(\Tilde{A}(r))$.
Also, let us define $\Bar{\mu} := \max \limits_{v} \sup \limits_{x \in [0,1]} \mu_v(x)$.

\section{Algorithm}
\label{sec:Algorithm}
In this section, we present a greedy algorithm for recovering the underlying directed graph $G$ from observations $\{(N_v(t), M_v(t)): v \in V\}_{t=0}^{T-1}$, which is an adaptation of the RecGreedy algorithm introduced for i.i.d. graphical model learning \cite{Ray}. Our main technical contribution is in providing sample complexity guarantee for our Markovian dynamics. We prove that the algorithm can recover the true graph with high probability using observations over $O(\log |V|)$ time steps.
 
 The algorithm uses estimates of directed conditional entropy, $\Hhat(v_+|v,Q)$ for $Q\subset V$.
 For each node $v$, the algorithm starts with empty neighborhood estimates $\hat{U}(v)$ and $\hat{T}(v)$; at every step, it chooses the node that produces the maximum reduction in the directed conditional entropy, i.e., the node for which $\Hhat(v_+|v,\hat{U}(v))-\Hhat(v_+|v,\hat{U}(v),u)$ is maximum. If this reduction is greater than a threshold, the corresponding node is added to the neighborhood estimate ($\hat{U}(v)$). Nodes are added until no new node results in a reduction greater than the threshold. 
 
 When the threshold is chosen as $\epsilon/2$, by Lemma~\ref{lem:tau_modified}, the last node added to $\hat{U}(v)$ is always a true neighbor of $v$. Only the last node is added to the final neighborhood estimate $\hat{T}(v)$. The whole process is repeated, after initializing $\hat{U}(v)=\hat{T}(v)$. When $\hat{T}(v)$ becomes equal to the true neighborhood, no new node can be added to $\hat{T}(v)$, and the algorithm stops at this point.

The sample complexity of this algorithm mainly depends on the dimension of the probability distribution that needs to be estimated from the data. As the algorithm starts with the smallest subsets and stops when there is no further drop in entropy, the dimensionality of the estimated distributions is bounded. 

To derive the algorithm's sample complexity, we introduce a notation $\Pmax$ to denote the maximum size of the subset $P$ from Lemma~\ref{lem:tau_modified}. We will prove the upper bound on the value of $\Pmax$ in section~\ref{sec:proofs}. This means that the error limit of $\epsilon/4$ is only required for estimating entropies conditioned on subsets of size up to $\Pmax$. The support of random variable $Y_v(t)$ is denoted by $\chi$ and $|\xi|:=|\chi|^{\Pmaxx}$ is the maximum size of the joint distribution that needs to be estimated. This, along with bounds on the evolution of the Markov dynamics, results in logarithmic sample complexity, which is formally stated in the Theorem ~\ref{thm:2}.

\begin{theorem}
\label{thm:2}
The RecGreedy($\epsilon$) algorithm recovers the true influence graph with probability at least $1-\gamma$ using data for 
\begin{align*}
 T \geq d+ \frac{\left(\log(|V|)\cdot(\Pmax+1) +\log(2|\xi|/\gamma)\right)\cdot |\xi|^2}{(1-2(\Bar{\mu} + L) |\rho(\Tilde{A})|^{\frac{1}{d}})\delta^2}   
\end{align*}
 time steps where $|\xi|:=|\chi|^{\Pmaxx}$, $|\chi| \leq (\frac{\Mbar(\Mbar+1)}{2}+2), |P|_{max} \leq  \frac{2\log(\frac{\Mbar(\Mbar+1)}{2}+2)}{\epsilon}+1$, $\delta = \frac{\epsilon^2}{8|\xi|}$, if $2(\Bar{\mu} + L) |\rho(\Tilde{A})|^{\frac{1}{d}}< 1$ and Assumption~\ref{assm:3_epsilon} holds. The computational complexity is $\mathcal{O}(|V|)$.
\end{theorem}

%----------------------------------------------------------
\begin{algorithm}
\caption{RecGreedy($\epsilon$) Algorithm for $d$-Markov}\label{alg_recgreedy}
 \hspace*{\algorithmicindent} \textbf{Input}: $\{N_v(t) : v \in V\}_{t=0}^{T-1}, \{M_v(t) : v \in V\}_{t=0}^{T-1}, $ \\
 \hspace*{\algorithmicindent} \textbf{Output}: ${\hat{T}(v): v \in V}$
\begin{algorithmic}[1]
 \For{$v \in V$}
    \State $\hat{T}(v), \leftarrow \phi$
    \Repeat
        \State $\hat{U}(v) \leftarrow \hat{T}(v)$
        \State $\text{LastNode} \leftarrow \phi$
        
            \Repeat
            \State{\vspace{-0.7cm}
            \begin{align*}
            u &= \arg \max \limits_{k \in V \setminus \hat{U}(v)} \hat{H}(v_+|v,\hat{U}(v)) \\ &- \hat{H}(v_+|v,\hat{U}(v),k) 
            \end{align*}}
            \State {$\Delta_u =  \hat{H}(v_+|v,\hat{U}(v)) - \hat{H}(v_+|v,\hat{U}(v),u) $}
              \If{$\Delta_u > \epsilon/2$}             
            
                    \State $\hat{U}(v) \leftarrow \hat{U}(v) \cup \{u\}$
                    \State $\text{LastNode} \leftarrow u$
                
              \Else
                \State $\hat{T}(v) \leftarrow \hat{T}(v) \cup \text{LastNode}$
              \EndIf 
            
        \Until New node is not added to $\hat{U}(v)$
    \Until New node is not added to $\hat{T}(v)$
\EndFor
            
\end{algorithmic}
\end{algorithm}

%-----------------------------------------------------------

\section{Proofs}
\label{sec:proofs}

Our main results are stated in terms of the quantity $\Pmax$, which denotes the maximum size of the subset of nodes the algorithm considers in any of its iterations. This is because the algorithm is based on conditional entropy estimation from the data; the accuracy of entropy estimation depends on the dimensionality of the distribution under consideration. Therefore, we structure our proof as follows: In Lemma~\ref{lem:5_Pmax_rec}, we state the value of $|P|_{max}$ for Algorithm~\ref{alg_recgreedy}.  Then, in Lemma~\ref{lem:3_Tmain}, we state the number of samples that are sufficient to achieve the entropy estimation accuracy stated in Lemma~\ref{lem:tau_modified}. 

After stating these results, we provide our proof for Theorem~\ref{thm:2}. The detailed proofs of Lemmas~\ref{lem:CondEntMarkov}-~\ref{lem:3_Tmain} are provided in section ~\ref{sec:detailed_proofs}. A novel contribution of this work is the proof technique used for Lemma~\ref{lem:3_Tmain}, where we use an improved Markov Hoeffding inequality from \cite{cineq} and obtain a bound on the second largest eigenvalue of the concerned Markov chain. The bound on the second eigenvalue was obtained by upper and lower bounding total variation distance with respect to stationary distribution and then by comparing these bounds. This proof may be of independent mathematical interest. We also give an indirect proof of the non-degeneracy assumption in \cite{Ray} for our particular $d$-Markov chain under certain assumptions (Lemma~\ref{lem:CondEntMarkov}).

\begin{lemma}
\label{lem:5_Pmax_rec}
    The maximum possible size of the set $P \subset V$ for which the Algorithm ~\ref{alg_recgreedy} evaluates directed conditional entropy has an upper bound given by $$|P|_{max} \leq \frac{2\log(\frac{\Mbar(\Mbar+1)}{2}+2)}{\epsilon}+1$$ 
\end{lemma}

%------------------------------------------------

\begin{lemma}
\label{lem:3_Tmain}
If Assumption~\ref{assm:3_epsilon} holds and the number of time steps, $T \geq d + \frac{\left(\log(|V|)\cdot(\Pmax+1) +\log(2|\xi|/\gamma)\right)\cdot |\xi|^2}{(1-2(\Bar{\mu} + L) |\rho(\Tilde{A})|^{\frac{1}{d}})\delta^2}$
where $\delta=\frac{\epsilon^2}{8|\xi|}$, then $|\Hhat(v_+|v,P) - H(v_+|v,P)| \leq \frac{\epsilon}{4}$ with probability at least $1 - \gamma$ . 
\end{lemma}
Proof is given in section~\ref{sec:detailed_proofs}

\begin{proof}[Proof of Theorem~\ref{thm:2}:]
To recover the graph correctly, the RecGreedy($\epsilon$) algorithm requires the following conditions: a) If the greedy step encounters a neighbor $u \in \mathcal{N}_v$ of a node $v \in V$, it will be added. b) The last node added by the greedy step is always a neighbor. From Lemma~\ref{lem:tau_modified}, we see that these conditions are satisfied if the error in entropy estimation is limited to $\epsilon/4$. From Lemma~\ref{lem:3_Tmain}, we get the number of samples sufficient to limit the entropy estimation error to $\epsilon/4$ with probability at least $1 - \gamma.$
Using Lemma~\ref{lem:5_Pmax_rec}, we substitute the value of $\Pmax  \leq \frac{2\log(\frac{\Mbar(\Mbar+1)}{2}+2)}{\epsilon}+1$ in  Lemma~\ref{lem:3_Tmain} and obtain the sample complexity stated in Theorem~\ref{thm:2}.

Proof of computational complexity: For every node, the greedy step runs at most $\Pmax$ times. Therefore, the computational complexity is $\mathcal{O}(|V|)$.

\end{proof}

%=====================================================

\subsection{Detailed Proofs}
\label{sec:detailed_proofs}

\begin{proof}[Proof of Lemma~\ref{lem:CondEntMarkov}]
$H(v_+|v,Q) - H(v_+|v,Q,u)$ can be written as 
\begin{align*}
    & H(v_+|v,Q) - H(v_+|v,Q,u) \nonumber \\
    & = \E_{v,Q,u,v_+} \left[\log \left( \frac{p(Y_{v_+}|\Yd_{v,Q,u})}{p(Y_{v_+}|\Yd_{v,Q})} \right) \right]
\end{align*}
Thus we see that $H(v_+|v,Q) - H(v_+|v,Q,u)$ is actually the expected K-L divergence between two distributions over $Y_{v_+}$, where expectation is taken with respect to $\Yd_v$, $\Yd_Q$ and $\Yd_u$, i.e.,
\begin{align*}
    & H(v_+|v,Q) - H(v_+|v,Q,u)\\
    & =\E_{u,v,Q}\left[D_{y_{v_+}}\left( {p(y_{v_+}|\yd_{v,Q,u})}||{p(y_{v_+}|\yd_{v,Q})} \right) \right]
\end{align*}
Using Pinsker's inequality \cite{pinskers}, for all $\yd_{v,Q,u}$,  %there exists $y_u$ such that
\begin{align}
    & \quad D_{y_{v_+}}\left( {p(y_{v_+}|\yd_{v,Q,u})}||{p(y_{v_+}|\yd_{v,Q})} \right)\nonumber \\
    & \geq 2\| p(y_{v_+}|\yd_{v,Q,u})-p(y_{v_+}|\yd_{v,Q})\|_1^2 \nonumber \\
    %& \geq 2|p(y_{v_+}|\yd_{v,Q,u})-p(y_{v_+}|\yd_{v,Q})|^2 
    %& \geq \epsilon'.
    \end{align}
%Here, the last step follows from Assumption~\ref{assm:3_epsilon}.
Therefore, we have 
\begin{align}
    & H(v_+|v,Q) - H(v_+|v,Q,u) \nonumber \\
    & =  \E_{v,Q} \E_{u}[D_{y_{v_+}}\left( {p(y_{v_+}|\yd_{v,Q,u})}||{p(y_{v_+}|\yd_{v,Q})} \right)| v, Q] \nonumber \\
    & \geq \E_{v,Q}   \E_{u}[\| p(y_{v_+}|\yd_{v,Q,u})-p(y_{v_+}|\yd_{v,Q})\|_1^2)| v, Q]. \nonumber \\
    %& \geq \E_{v,Q}   \E_{u}[| p(y_{v_+}|\yd_{v,Q,u})-p(y_{v_+}|\yd_{v,Q})|_1^2)| v, Q]. \nonumber \\
    & \geq \E_{v,Q} \big[ \epsilon'  \prob_{u}(\| p(y_{v_+}|\yd_{v,Q,u})-p(y_{v_+}|\yd_{v,Q})\|_1^2 > \epsilon' ) \nonumber \\ 
    & \quad \mid v, Q]\big]
\end{align}

By Assumption~\ref{assm:3_epsilon}, there exists a $y_u$ such that $\| p(y_{v_+}|\yd_{v,Q,u})-p(y_{v_+}|\yd_{v,Q})\|_1^2) > \epsilon'.$ For $M_u(t)=M$, we prove that the probability of such a $y_u$ is lower bounded i.e., we prove that $\inf_{y_u, y_v, y_Q} \prob(Y_u=y_u|y_{v,Q})$ is lower bounded, which would prove a lower bound on  $H(v_+|v,Q) - H(v_+|v,Q,u)$.

$\prob(Y_u=y_u|y_{v,Q})$ is a binomial distribution with parameter $M$ and $\E[X_u|y_{v,Q}]$. The minimum probability of any binomial with parameter $M$ and $a$ is $\min\{a^M,(1-a)^M\}$.
Now,

\begin{align*}
    X_u(t)=(1-\alpha_u)((1-\beta)Z_u(t) + \beta l_v) \nonumber \\ + \alpha_v\sum_{u \in \mathcal{N}_v \bigcup {v} }\sum_{r=0}^{(d-1)}a_{uv}^{(r)}\{\frac{N_v(t-1)}{M_v(t-1)}\}
\end{align*}

and hence, $a=\inf_{y_{v,Q}}\E[X_u|y_{v,Q}]=(1-\alpha_u)((1-\beta) z_u + \beta l_v)+\alpha_u \theta_{u}$, where $0\leq \theta_u \leq 1$. Note that since $0<z_u<1$, $0<\alpha<1$ we have $0<a<1$ and hence,
$\inf_{y_u, y_v, y_Q} \prob(Y_u=y_u|y_{v,Q}) = \min \{a^M,(1-a)^M\}$.

Thus under Assumption~\ref{assm:3_epsilon}, $H(v_+|v,Q) - H(v_+|v,Q,u) >\epsilon := \epsilon'\cdot \min \{a^M,(1-a)^M\}$ when $u$ is a neighbor of $v$.

Also by Lemma ~\ref{lem:1_Markov}, if $Q$ contains $\mathcal{N}_v$ and $u \notin \mathcal{N}_v$, $p(y_{v_+}|\yd_{v,Q,u})=p(y_{v_+}|\yd_{v,Q})$ for all $(y_{v_+},\yd_{v,Q,u})$. Thus $H(v_+|v,Q) - H(v_+|v,Q,u)=0$ when $Q$ contains the neighborhood. 
    
\end{proof}

\begin{proof}[Proof of Lemma \ref{lem:tau_modified}]

The proof follows on similar lines as \cite{Ray}. If $|\Hhat(v_+|v,Q) - H(v_+|v,Q)| < \epsilon''$ for $v \in V \text{ and } Q \subseteq V \text{ such that } |Q| < |P|_{max}$, then $$ H(v_+|v,Q)-\epsilon''  < \Hhat(v_+|v,Q) < \epsilon'' + H(v_+|v,Q)$$ Similar condition holds for $\Hhat(v_+|v,Q,u), u \in V.$ When we combine these two conditions and apply Lemma \ref{lem:CondEntMarkov}, we conclude that $\Hhat(v_+|v,Q) - \Hhat(v_+|v,Q,u)$ is greater than $\epsilon - 2\epsilon''$ if $u$ is a neighbor and is less than $2\epsilon''$ if $\mathcal{N}_v \subseteq Q \text{ and } u \notin \mathcal{N}_v$.

Substituting $\epsilon'' = \epsilon/4$ proves Lemma \ref{lem:tau_modified}. 
\end{proof}

\begin{proof}[Proof of Lemma~\ref{lem:5_Pmax_rec}] 
In the context of this work, the algorithm starts with an initial entropy of $\Hhat(v_+|v)$, which is reduced by $\epsilon/2$ every time a neighbor is added. The maximum number of nodes that can be estimated as neighbors is (similar to \cite{Ray}) $\frac{2\Hhat(v_+|v)}{\epsilon}$. Therefore, the maximum number of steps $|P|_{max}$ is $\frac{2\Hhat(v_+|v)}{\epsilon}+1$. We know that $\Hhat(v_+|v) \leq \log|\chi|$ where $|\chi|$ is the support of $y_{v_+}$. Here, $|\chi| \leq (\frac{\Mbar(\Mbar+1)}{2}+2)$.Thus, $|P|_{max} \leq \frac{2\log(\frac{\Mbar(\Mbar+1)}{2}+2)}{\epsilon}+1$.
\end{proof}

Next, in order to prove Lemma~\ref{lem:3_Tmain}, we first state a few results in Lemmas~\ref{lem:7_infoth}-\ref{lem:TV_upperbound}. 
\begin{lemma}
    \label{lem:7_infoth}
    $|\Hhat(v_+|v,P) - H(v_+|v,P)| < \epsilon/4$ if $\|\phat(y_{v_+},\yd_{v,P})-p(y_{v_+},\yd_{v,P})\|_1 < \delta$, where $\phat(\cdot)$ is the estimated probability distribution, $p(\cdot)$ is the true probability distribution and $\delta = \frac{\epsilon^2}{8|\xi|}$
\end{lemma}

 \begin{proof} The entropy estimation error depends on the error in the estimation of probability as \cite{infotheory} 
 %\begin{fleqn}%[\noindent]
\begin{equation}
 \begin{split}
 %\begin{flalign}
        |\Hhat(v_+,v,P)-H(v_+,v,P)| 
        \leq  -\|\phat(y_{v_+},\yd_{v,P})- \\ p(y_{v_+},\yd_{v,P})\|_1
        \times \log\frac{\|\phat(y_{v_+},\yd_{v,P})-p(y_{v_+},\yd_{v,P})\|_1}{|\xi|}
       % \end{flalign}
  \end{split}
\end{equation}
%\end{fleqn}

If $\|\phat(y_{v_+},\yd_{v,P})-p(y_{v_+},\yd_{v,P})\|_1 \leq \delta$, then $|\Hhat(v_+,v,P)-H(v_+,v,P)| \leq \sqrt{\delta |\xi|}$. Since the conditional entropy is computed from the estimated joint entropy as $\Hhat(v_+|v,P) = \Hhat(v_+,v,P) - \Hhat(v,P)$, we have 
$|\Hhat(v_+|v,P) - H(v_+|v,P)| \leq 2\sqrt{\delta|\xi|}= \epsilon/4$.
\end{proof}

Next, let $\lambda^*$ denote the second largest eigenvalue of the Markov chain $\{\Yd(t)\}$. Lemma~\ref{lem:8_tp} states the sample complexity in terms of $\lambda^*$ using the Markov Hoeffding inequality from \cite{cineq}. Later, in Lemmas~\ref{lem:9_lambda} and \ref{lem:TV_upperbound}, we upper bound the absolute value of $\lambda^*$.
\begin{lemma}
    \label{lem:8_tp}
     If $|\lambda|^*<1$ and the number of data points $T$ available for probability estimation is greater than $d+\frac{\left(\log(|V|)\cdot(\Pmax+1) +\log(2|\xi|/\gamma)\right)\cdot |\xi|^2}{(1-|\lambda^*|)\delta^2}$, then $\|\phat(y_{v_+},\yd_{v,P})-p(y_{v_+},\yd_{v,P})\|_1 < \delta$ with probability at least $1-\gamma$.
\end{lemma}
\begin{proof} Consider the Markov chain $\mathbf{Z}(t)=(\Yd(t),\Yd(t-1))$. We define a function $f(\mathbf{Z}(t)) = \frac{\sum_{t=d}^{T-1} \mathbb{I}\left((Y_{v}(t),\Yd_{v,P}(t-1))=(y_{v_+},\yd_{v,P})\right)}{T-d},$ where $\mathbb{I}$ is the indicator function. Using the concentration inequality stated in \cite{cineq} for $\mathbf{Z}(t)$ and $f$, we get
\begin{align*}
&~~~\prob(|f(\mathbf{Z}(t)) - \mathbb{E}f(\mathbf{Z}(t))|>\delta_1) \\
& \leq 2\exp\left({-\frac{(1-|\lambda^*|)2(T-d)^2\delta_1^2}{(1+|\lambda^*|)(T-d)}}\right)
\end{align*}
By applying union bound on $(y_{v_+},\yd_{v,P})$, we have
\begin{align*}
&~~ \prob(\|\phat(y_{v_+},\yd_{v,P})-p(y_{v_+},\yd_{v,P})\|_1>\delta) \\ & \leq 2|\xi|\exp\left({-\frac{(1-|\lambda^*|)2(T-d)(\delta/|\xi|)^2}{(1+|\lambda^*|)}}\right),
\end{align*}
which implies
\begin{align*}
&\prob(\|\phat(y_{v_+},\yd_{v,P})-p(y_{v_+},\yd_{v,P})\|_1>\delta) \leq \\
& 2|\xi|\exp\left({-\frac{2(1-|\lambda^*|)(T-d)\delta^2}{(1+|\lambda^*|).(|\xi|^2)}}\right)
\end{align*}

Applying union bound on $v$ and $P$,
\begin{align*}
& \prob(\|\phat(y_{v_+},\yd_{v,P})-p(y_{v_+},\yd_{v,P})\|_1>\delta) \leq \\ 
& 2|V| {{|V|} \choose{\Pmax}} |\xi|\exp\left({-\frac{2(1-|\lambda^*|)(T-d)\delta^2}{(1+|\lambda^*|)(|\xi|^2)}}\right)
\end{align*}
This probability is less than or equal to $\gamma$ when 
\begin{align*}
\gamma \geq 2|V| \cdot  {{|V|} \choose{\Pmax}} |\xi| \exp\left({-\frac{2(1-|\lambda^*|)(T-d)\delta^2}{(1+|\lambda^*|)(|\xi|^2)}}\right)
\end{align*}
%---------------
\begin{align*}
\implies & \gamma \geq 2|V|\cdot |V|^{\Pmax} \\& ~~~\times |\xi|\exp\left({-\frac{2(1-|\lambda^*|)(T-d)\delta^2}{(1+|\lambda^*|)|\xi|^2}}\right)
\end{align*}
%------------

 Therefore, 
\begin{align*}
 T \geq d + \frac{(|\xi|^2)(\log(|V|)(\Pmax+1) +\log(\frac{2|\xi|}{\gamma}) )}{(1-|\lambda^*|)\delta^2}
\end{align*}
\end{proof}

\begin{lemma}
    \label{lem:9_lambda}
    $2c|\lambda^*|^t \leq \frac{1}{2} \max \limits_{\yd_0} \|\mathcal{P}^t(\yd_0,.)-\bm{\pi}\|_1$, where  $\mathcal{P}$ is the transition matrix, $\mathcal{P}^t(\yd_0,.)=\prob(\Yd(t)|\Yd(0)=\yd_0)$, $\bm{\pi}$ is the stationary distribution and $c$ is a constant independent of $t$ and $|V|$.
\end{lemma}

\begin{proof}

Consider the total variation distance  $\mbox{dist}(t):=\frac{1}{2} \max \limits_{\yd_0} \|\mathcal{P}^t(\yd_0,.)-\bm{\pi}\|_1$. 

\begin{align*}
\tm^t(\yd_0,.)-\bm{\pi} &= \mathbf{e}_{0}\left(\mathbbm{1}^T\bm{\pi} + \sum \limits_{i=2}^N \lambda_i^t \w_i\z_i^T\right) - \bm{\pi} \nonumber \\
&= \mathbf{e}_{0}\left(\sum \limits_{i=2}^N \lambda_i^t \w_i\z_i^T\right),
\end{align*}

where $\mathbf{e}_{0}$ is the initial state vector given the initial state is $\yd_0$, $\w_i$ and $\z_i$ are the right and left eigenvectors of $\tm$ and $N$ is the number of states. Now, 

\begin{align*}
    &\|\mathcal{P}^t(\yd_0,.)-\bm{\pi}\|_1 = \left\|\eo\left(\sum \limits_{i=2}^N \lambda_i^t \w_i\z_i^T\right)\right\|_1\nonumber \\
    &= |\lambda^*|^t \left\|\eo\left(\w^*\z^{*T} + \sum_{\substack{i=3}}^N  \left(\frac{\lambda_i}{\lambda^*}\right)^t \w_i\z_i^T\right)\right\|_1 \nonumber \\
    %& \ge c |\lambda^*|^t,
\end{align*}

For all sufficiently large $t$, $\left\|\eo\sum_{\substack{i=1 \\ i \neq i^*}}^N  \left(\frac{\lambda_i}{\lambda^*}\right)^t \w_i\z_i^T\right\|_1$ is exponentially smaller than $\left\|\eo\w^*\z^{*T}\right\|$ since $\lambda_i < \lambda^*$ for each $i\geq 3$. Since $\left\|\eo\left(\w^*\z^{*T} + \sum_{\substack{i=3}}^N  \left(\frac{\lambda_i}{\lambda^*}\right)^t \w_i\z_i^T\right)\right\|_1$ is lower bounded by $\left\|\eo\w^*\z^{*T}\right\|_1 - \left\|\eo\sum_{\substack{i=1 \\ i \neq i^*}}^N  \left(\frac{\lambda_i}{\lambda^*}\right)^t \w_i\z_i^T\right\|_1$, we get that for all sufficiently large $t$, $\left\|\eo\left(\w^*\z^{*T} + \sum_{\substack{i=3}}^N  \left(\frac{\lambda_i}{\lambda^*}\right)^t \w_i\z_i^T\right)\right\|_1 \geq \frac{1}{2}\left\|\eo\w^*\z^{*T}\right\|_1=c$. Thus, for sufficiently large $t_1$ we get

\begin{equation}
\label{eq:d_t_lambda}
    \mbox{dist}(t) \geq 2c|\lambda^*|^t  \quad \text{for all } t>t_1
\end{equation}
\end{proof}

\begin{lemma}
    \label{lem:TV_upperbound}
     $\mbox{dist}(t) \le c' \left(2(\Bar{\mu} + L)|\rho(\Tilde{A})|^{\frac{1}{d}}\right)^t$ for some constant $c'$.
\end{lemma}

\begin{proof}
We know that 
\begin{align*}
 \mbox{dist}(t) \leq \Bar{ \mbox{dist}(t)}:= \frac{1}{2} \max \limits_{\yd_0, \ydd_0} \Big\| \prob(\Yd(t) | \Yd(0) =\yd_0) \nonumber \\ - \prob(\Yd(t) | \Yd(0) =\ydd_0) \Big\|_{1}
\end{align*} 
Now we consider a coupling of the Markov chain, the process $\{(\Yd(t),\Ydd(t))\}$ with the initial conditions $\Yd(0)=\yd(0), \Ydd(0)=\ydd(0)$. We generate these two processes as follows:

At time $t$, we only need to generate $\Y(t)$ and $\Ydash(t)$; the rest of $d-1$ terms are taken from $\Yd(t-1)$ and $\Ydd(t-1)$. To generate $\Y(t)$ and $\Ydash(t)$, we first generate $\X(t)$ and $\Xdash(t)$ from $\Yd(t-1)$ and $\Ydd(t-1)$ according to eq\eqref{eq:1_model}.

Now,  we denote introduce a random variable $R_v(t)$ to denote $\text{Poisson}(\min\{\mu_v(X_v(t)),\mu_v(X'_v(t))\})$ and $S_v(t)$ to denote $\text{Poisson}(\max\{\mu_v(X_v(t)),\mu_v(X'_v(t))\} - \min\{\mu_v(X_v(t)),\mu_v(X'_v(t))\}) .$ If the minimum among $\mu_v(X_v(t))$ and $\mu_v(X'_v(t))\}$ is $\mu_v(X_v(t))$, then generate $M_v(t) = \min \{R_v(t), \Mbar\} + 1 $ and $M'_v(t) = \min \{R_v(t) + S_v(t), \Mbar\} + 1$. 

Else, if If the minimum among $\mu_v(X_v(t))$ and $\mu_v(X'_v(t))\}$ is $\mu_v(X'_v(t))$, then generate $M'_v(t) = \min \{R_v(t), \Mbar\} + 1 $ and $M_v(t) = \min \{R_v(t) + S_v(t), \Mbar\} + 1$. 

To generate $\Yd, \Ydd$, we generate $\min \{M_v(t), M'_v(t)\}$ number of $\text{Unif}[0,1]$ random variables (the $i^{th}$ uniform random variable is denoted by $U_i$). Then compare each of these uniform random variables with $X_v(t)$ to generate observed binary variables for the process $\Yd(t)$ (if $X_v(t)$ is less than or equal to the uniform variable, it's a `1'). Then, compare the same uniform random variables with $X'_v(t)$ to generate the observed binary variables for $\Ydd(t).$ Here, $\min \{M_v(t), M'_v(t)\}$ uniform random variables are common to both processes. 
Then, generate the remaining binary variables using a new set of $\max \{M_v(t), M'_v(t)\} - \min \{M_v(t), M'_v(t)\}$ uniform random variables.

Now we will consider the distance $\Bar{d}(t)$ between these two processes. $$\Bar{ \mbox{dist}(t)} \leq \prob(\tau_{couple} > t | \init),$$
For ease of notation, from here on, we denote the initial condition event $\{\init\}$ by $\0$
where $\tau_{couple} := \min \{t:\Yd(t) = \Ydd(t)\}$.
\begin{align*}
    \prob(\tau_{couple} > t\mid\0) = \prob(\Yd(t) \neq \Ydd(t)\mid\0) 
\end{align*}
$\Yd(t)$ and $\Ydd(t)$ are not equal when $\Y_v(k) \neq \Ydash_v(k)$ for at least one $(k,v)$ such that $ k \in \{t,t-1,\hdots,t-d+1\}, \text{ and } v \in \{1,2,\hdots |V|\}$. Therefore, we apply the union bound on $k$ and $v$ to obtain eq(\ref{eq:prob_t_k}) 

\begin{align}
   & \prob(\tau_{couple} > t\mid\0) \nonumber \\
    &= \prob \Big( \bigcup\limits_{r=0}^{d-1}\{\Y(t-r) \neq \Ydash(t-r)\}\mid \nonumber  \0 \Big) \nonumber \\
    & \leq \sum\limits_{r=0}^{d-1} \prob( \Y(t-r) \neq \Ydash(t-r)\mid \nonumber \0) \nonumber \\
    &= \sum\limits_{k=t-d+1}^t \prob(\Y(k) \neq \Ydash(k)\mid \nonumber \0) \nonumber \\
    & \leq \sum \limits_{v=1}^{|V|} \sum\limits_{k=t-d+1}^t \prob(Y_v(k) \neq Y'_v(k)\mid \label{eq:prob_t_k} \0) 
\end{align}

The event $Y_v(k) \neq Y'_v(k)$ can be partitioned into events based on whether the total number of binary random variables is equal for the two processes. Therefore,
$\prob(Y_v(k) \neq Y'_v(k)\mid\0)$ can be written as 
\begin{align}
\label{eq:split_tcouple}
    & \prob(Y_v(k) \neq Y'_v(k)\mid\0) \nonumber \\ &= \prob(M_v(k) = M'_v(k),Y_v(k) \neq Y'_v(k) |\0) \nonumber \\ &+ \prob(M_v(k) \neq M'_v(k),Y_v(k) \neq Y'_v(k) |\0)
\end{align}

Now, we consider the first term of eq~\eqref{eq:split_tcouple}. The joint probability $\prob (M_v(k) = M'_v(k),Y_v(k) \neq Y'_v(k) | \cdot)$ is the same as $\prob (M_v(k) = M'_v(k),N_v(k) \neq N'_v(k) | \cdot)$, since $Y_v(k) = N_v(k) / M_v(k)$. Then, we apply the law of total probability with respect to the smallest disjoint events generated by $\{\Yd(k-1),\Ydd(k-1)\}$ to obtain $\E_{\Yd(k-1), \Ydd(k-1)} [\prob(N_v(k) \neq N'_v(k),M_v(k) = M'_v(k) \mid \Yd(k-1), \Ydd(k-1);\0)].$ This is an expectation over the distribution of $\{\Yd(k-1),\Ydd(k-1)\}$ given the initial condition $'\0'$. 

Also, any joint probability is less than or equal to conditional probability. Therefore, the first term of eq~\eqref{eq:split_tcouple} turns out to be less than or equal to $\E_{\Yd(k-1), \Ydd(k-1)} [\prob(N_v(k) \neq N'_v(k)\mid M_v(k) = M'_v(k) ;\Yd(k-1), \Ydd(k-1);\0)].$

\begin{align}
\label{eq:Expectation_k-1_N_v}
    &  \prob \big(M_v(k) = M'_v(k),Y_v(k) \neq Y'_v(k) |  \0 \big) \nonumber \\
    &  \leq \E_{\Yd(k-1), \Ydd(k-1)} [\prob(N_v(k) \neq N'_v(k) \nonumber \\
    & \quad \mid M_v(k) = M'_v(k);  \0;\Yd(k-1),\Ydd(k-1))]
\end{align}

Now, we analyze the term $\prob(N_v(k) \neq N'_v(k) \mid M_v(k) = M'_v(k); \0; \Yd(k-1),\Ydd(k-1)).$ Since  $M_v(k)=M'_v(k)$, all the uniform random variables are common to both processes. $N_v(k)$ and $N'_v(k)$ will not be equal when $X_v(k) \leq U_i < X'_v(k)$ or $X'_v(k) \leq U_i < X_v(k)$ for at least one out of the $M_v(k)$ uniform random variables $\{U_i\}_{i=1}^{M_v(k)}$. Therefore, we can apply union bound on the uniform random variables and use the law of total probability over $M_v(k).$ Further, the probability that a uniform random variable lies in $[X_v(k), X'_v(k))  \cup [X'_v(k),X_v(k))$ is given by $|X_v(k) - X'_v(k)|$. Finally, using the model eq~\eqref{eq:1_model}, we can write this expression in terms of $\Yd(k-1),\Ydd(k-1).$

\begin{align*}
    & \quad \prob(N_v(k) \neq N'_v(k) \mid \\ 
    & \quad M_v(k) = M'_v(k); \Yd(k-1),\Ydd(k-1);  \0) \\
    & \leq \E_{M_v(k)} \Big[M_v(k) \prob_U(\{U \in  [X_v(k), X'_v(k))  \cup [X'_v(k),X_v(k))  \nonumber \\
    & \quad \mid  M_v(k) = M'_v(k); \0; \Yd(k-1),\Ydd(k-1)) \Big] \nonumber \\
    & \leq \max \limits_{X_v(k)} \E_{M_v(k)}[M_v(k)] \lvert X_v(k)-X'_v(k) \rvert \nonumber \\
    &= \Bar{\mu} \left[ \alpha_v\left| \sum_{u \in \mathcal{N}_v \cup \{v\}} \mathbf{a}^{(d)}_{uv}\Yd(k-1)- \mathbf{a}^{(d)}_{uv}\Ydd(k-1) \right|\right]
\end{align*}

Thus, we have  
\begin{align}
    &\prob \big(M_v(k) = M'_v(k),Y_v(k) \neq Y'_v(k) | \0 \big) \nonumber \\
    &\leq \Bar{\mu}\E_{\Yd(k-1), \Ydd(k-1)}\Bigg[ \alpha_v\Big| \sum_{u \in \mathcal{N}_v \cup \{v\}} \mathbf{a}^{(d)}_{uv}\Yd(k-1) \nonumber \\
    &- \mathbf{a}^{(d)}_{uv}\Ydd(k-1) \Big| \mid \0 \Bigg]
\end{align}

Now we derive an expression for the second term of eq~\eqref{eq:split_tcouple}, i.e., $\prob(M_v(k) \neq M'_v(k),Y_v(k) \neq Y'_v(k) |\0)$ in terms of $\Yd(k-1)$ and $\Ydd(k-1)$.

\begin{align*}
    &\prob(M_v(k) \neq M'_v(k),Y_v(k) \neq Y'_v(k) |  \0) \\
    &\leq \prob(M_v(k) \neq M'_v(k),\mid Y_v(k) \neq Y'_v(k); \nonumber \\ 
    & \Yd(k-1), \Ydd(k-1);\0)\\
    & =\E_{\Yd(k-1), \Ydd(k-1)} [\prob(M_v(k) \neq M'_v(k)\mid \nonumber \\
 & Y_v(k) \neq Y'_v(k); \Yd(k-1), \Ydd(k-1); \0) ]
\end{align*}

Since we have used two Poisson random variables  $R_v(t) =\text{Poisson}(\max \{\mu_v(X_v(k), \mu_v(X'_v(k)\})$ and $S_v(t)=\text{Poisson}(\max \{\mu_v(X_v(k), \mu_v(X'_v(k)\}-\min \{\mu_v(X_v(k), \mu_v(X'_v(k)\})$, we have

\begin{align*}
    &~~~\prob(M_v(k) \neq M'_v(k)\mid Y_v(k) \neq Y'_v(k); \nonumber \\
    &~~~~~~\Yd(k-1), \Ydd(k-1);\0)\\
    & =\prob(S_v(t) \neq 0 \mid Y_v(k) \neq Y'_v(k); \nonumber \\ 
    & ~~~\quad \Yd(k-1), \Ydd(k-1);\0 )\\
    & \leq \E[ 1 - \exp{(-|\mu_v(X_v(k))-\mu_v(X'_v(k))|)} \vert \\
    & ~~~~\Yd(k-1), \Ydd(k-1),  \0]\\
    &\leq \E[|\mu_v(X_v(k))-\mu_v(X'_v(k))|]\vert \\
    & ~~~~\Yd(k-1), \Ydd(k-1),  \0]\\
    &\leq L\E[|X_v(k)-X'_v(k)| \mid \Yd(k-1), \Ydd(k-1) \text{ and }\0]\\
    & = L \left| \alpha_v\left(\sum_{u \in \mathcal{N}_v \cup \{v\}}  \mathbf{a}^{(d)}_{uv}\Yd(k-1)- \mathbf{a}^{(d)}_{uv}\Ydd(k-1)\right) \right|\\
\end{align*}

Therefore,

\begin{align*}
    & \prob(Y_v(k) \neq Y'_v(k)\mid\0)\\
    & \leq (\Bar{\mu} + L)\E_{\Yd(k-1), \Ydd(k-1)}\Bigg[\alpha_v\Big\vert \sum_{u \in \mathcal{N}_v \cup \{v\}} \mathbf{a}^{(d)}_{uv}\Yd(k-1) \nonumber \\
    & \quad - \mathbf{a}^{(d)}_{uv}\Ydd(k-1) \Big\rvert  \mid \0\Bigg]
\end{align*}

Since $0 \leq Y_v(k), Y'_v(k) \leq 1,\E[|Y_v(k)-Y'_v(k)|] \leq \prob(Y_v(k) \neq Y'_v(k))$. Therefore,

\begin{align}
\label{eq_kminus1}
    & \prob(Y_v(k) \neq Y'_v(k)\mid\0) \nonumber \\
     & =(\Bar{\mu} + L) 
     \sum\limits_u\sum\limits_{r=0}^{(d-1)}a_{uv}^{(r)}\prob(Y_u(k-1-r) \neq Y'_u(k-1-r) \nonumber \\ 
     & \mid  \0)
\end{align}

Now, let $\eta_v(k) := \prob(Y_v(k) \neq Y'_v(k)\mid \0)$. Also, define a vector $\eta(k) := [\eta_1(k), \eta_2(k), \hdots ,  \eta_{|V|}(k) ]^T$ 

From eq ~\eqref{eq_kminus1} we have the following element wise inequality for $\eta(k)$:
\begin{align}
    \label{eq:beta}
    \eta(k) &\leq (\Bar{\mu} + L)\sum \limits_{r=0}^{d-1}\Tilde{A}(r)\eta(k-1-r) 
\end{align}
 We first choose an $r$ for which $\rho(\Tilde{A}(r))$ is the maximum and denote it by $r^*$. Next, we expand eq~\eqref{eq:beta} until we reach the initial condition, i.e. $\eta(0)= \eta(-1)= \hdots \eta(-d+1)= \mathbbm{1}.$ This expansion gives an inequality with a sum of powers of  $\Tilde{A}$ multiplied with the initial probability vector $\mathbbm{1}.$ Let the number of such terms at time $k$ be $\kappa(k).$ We have $\kappa(k) = \kappa(k-1) + \kappa(k-2) + \hdots + \kappa(k-d) \text{ for } k>d$ and $\kappa(k) = \sum \limits_{l=0}^{k-2} 2^{l}(d-(k-2-l)) + (d-r+1), \text{ for } k \in \{1, 2, \hdots d\}.$
 It follows after some calculations that any $\kappa(k)$ satisfying the above recursion also satisfies $\kappa(k) \leq 2^{k+d}$.  Also, the lowest power of $\Tilde{A}(r)$ for any $r$ that appears is $\Tilde{A}(r)^{k/d}$  Therefore, we have

\begin{align*}
    \eta(k) &\leq (\Bar{\mu}+L)^k\cdot {2}^{k+d}\Tilde{A}(r^*)^{k/d} \mathbbm{1}
\end{align*}

From the union bound in eq \eqref{eq:prob_t_k} we have
\begin{align}
    \mbox{dist}(t) &\leq d (\Bar{\mu}+L)^t\cdot {2}^{t+d}\|\Tilde{A}(r^*)^{t/d} \mathbbm{1}\|_1 
\end{align}
Let us write $\Tilde{A}(r^*)$ in terms of its eigenvalues and eigenvectors as $\Tilde{A}(r^*)^t = \sum_i \lambda_{A_i}^t \w_{A_i} \z_{A_i}^T = \rho(\Tilde{A})^t\left(\w_A \z_A^T + \sum_{i=2}^{|V|} (\frac{\lambda_{A_i}}{\rho(\Tilde{A}})^t \w_{A_i} \z_{A_i}^T\right)$. As $\lambda_{A_i} < \rho(\Tilde{A})$, $\left\|\sum_{i=2}^{|V|} (\frac{\lambda_{A_i}}{\rho(\Tilde{A}})^t \w_{A_i} \z_{A_i}^T )\mathbbm{1}\right\|$ is exponentially smaller than $\|\w_{A_i} \z_{A_i} \mathbbm{1}\|$ Therefore,
\begin{align}
\label{eq:coupling_lambda}
    & \mbox{dist}(t) \leq d(\Bar{\mu}+L)^t\cdot {2}^{t+d}\|\Tilde{A}^{t/d} \mathbbm{1}\|_1 \nonumber \\
    & \leq d(\Bar{\mu}+L)^t\cdot {2}^{t+d} |\rho(\Tilde{A})|^{t/d} \nonumber \\ 
    & \quad \times \left\|\left( \w_A \z_A^T + \sum_{i=2}^{|V|} (\frac{\lambda_{A_i}}{\rho(\Tilde{A})})^{t/d} \w_{A_i} \z_{A_i}^T\right)\mathbbm{1} \right\|_1 \nonumber \\
    & \leq d(\Bar{\mu}+L)^t\cdot {2}^{t+d} |\rho(\Tilde{A})|^{t/d}  \nonumber \\ 
    & \quad \times \left( \| \w_A \z_A^T \mathbbm{1}\|_1 +\left\|_1\sum_{i=2}^{|V|} (\frac{\lambda_{A_i}}{\rho(\Tilde{A})})^{t/d} \w_{A_i} \z_{A_i}^T \mathbbm{1} \right\|_1 \right)\nonumber \\
    & \leq d(\Bar{\mu}+L)^t\cdot {2}^{t+d} |\rho(\Tilde{A})|^{t/d} \cdot 2 \| \w_A \z_A^T \mathbbm{1}\|_1 \nonumber \\
    %& =(2(\Bar{\mu}+L))^t|\rho(\Tilde{A})|^{t/d}.c' \quad  \\
    & \text{for all $t>t_2$ where $t_2$ is sufficiently large} 
\end{align}
We get the final bound by choosing  $c'=2^{d+1} d \| \w_A \z_A^T \mathbbm{1}\|_1$.

\end{proof}

\begin{proof}[Proof of Lemma~\ref{lem:3_Tmain}]

 Combining eq\eqref{eq:d_t_lambda} and \eqref{eq:coupling_lambda}, we have 
 \begin{align*}
     &|\lambda|^* \leq (2c/c')^{(1/t)}\left(2(\Bar{\mu} + L)|\rho(\Tilde{A})|^{\frac{1}{d}}\right) \nonumber \\
     & \text{for all $t>\max\{t_1,t_2\}$} \nonumber \\
    & \therefore |\lambda|^* \leq  2(\Bar{\mu} + L)|\rho(\Tilde{A})|^{\frac{1}{d}}
 \end{align*} 
Next we substitute this upper bound on $|\lambda|^*$ in Lemma~\ref{lem:8_tp}, along with the appropriate value of $\delta$ (Lemma~\ref{lem:7_infoth}). This proves Lemma~\ref{lem:3_Tmain}.
\end{proof}

\section{Simulations}
\label{sec:NumericlResults}
Though our main contribution is analytical in nature, we perform multiple simulations to corroborate our analytical results. Interestingly, as is often the case, the simulations also lead to certain insights, which are analytically intractable.

\begin{figure*}[h]

\begin{subfloat}[\label{fig:d1m0}]
{\includegraphics[width=0.32\linewidth]{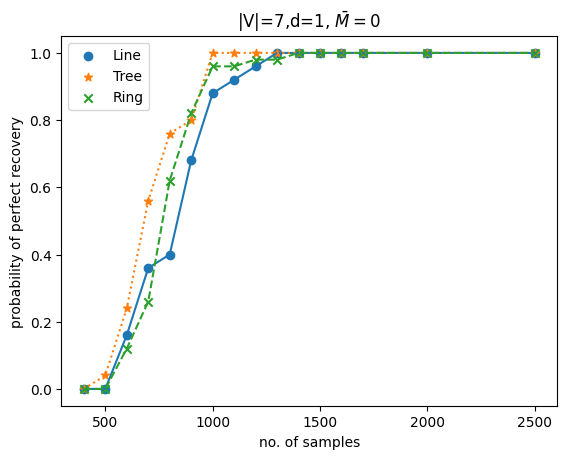}}
\end{subfloat}
\begin{subfloat}[]
    {\includegraphics[width=0.32\linewidth]{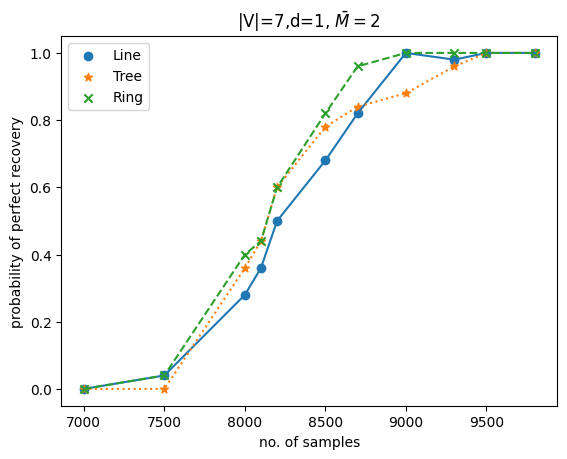}
    \label{fig:d1m2}}
\end{subfloat}
\begin{subfloat}[]
    {\includegraphics[width=0.32\linewidth]{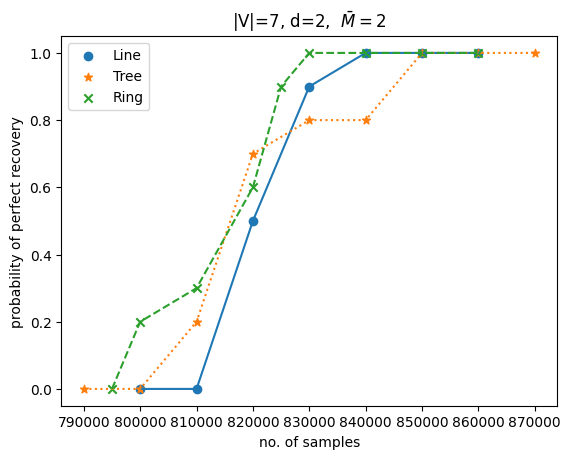}
    \label{fig:d2m2}}
\end{subfloat}
\caption{(a) Probability of perfect recovery for $|V|=7$,$d=1$, $\Mbar = 0$, (b)Probability of perfect recovery for $|V|=7,d=1$, $\Mbar=2$ and (c) Probability of perfect recovery for $|V|=7,d=2, \Mbar=2$ }
\end{figure*}

\begin{figure}[h]
    {\includegraphics[width=0.85\linewidth] {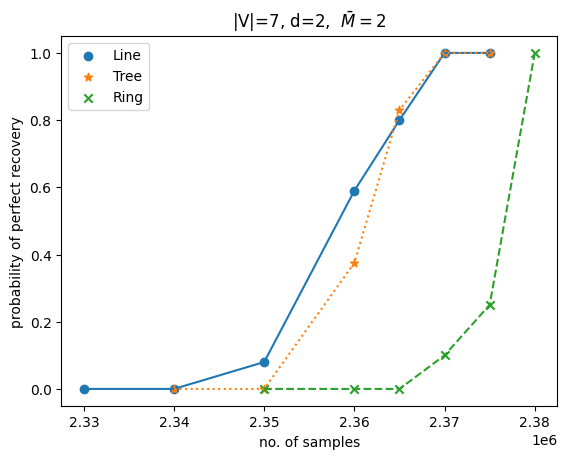}
    %\caption{(d)}
    \caption{ Probability of perfect recovery for $|V|=7,\Mbar = 2, d=2$ when $2(\Bar{\mu}+L)|\rho(\Tilde{A})|^{\frac{1}{d}} > 1$} 
    \label{fig:d2m2outside}}
\end{figure}

We generate data using Eq.~\ref{eq:1_model} for various values of $d$ and $\Bar{M}$ and run Algorithm~\ref{alg_recgreedy} on that data. We chose the parameters of the model as follows:  for all nodes $v$, $\beta=0.75$, $l_v=0.167$, $\alpha_v = 0.4$, $Z_v(t) \sim \text{Ber}(0.5)$. We simulate multiple dynamics for different memory ($d$) and $\Bar{M}$, and three different causal or influence graph structures ($G$): line, tree, and ring. 

We run Algorithm~\ref{alg_recgreedy} on data from all those different instants of the dynamics, plot the dependence of the probability of graph recovery on the number of samples, and observe the influence of $d$ and $\Bar{M}$ on those. The threshold $\epsilon/2$ in Algorithm~\ref{alg_recgreedy} has to be chosen appropriately in each case.

In Fig.~\ref{fig:d1m0}, we plot the probability of graph recovery against the number of samples used for $d=1$, $\Bar{M}=0$ (i.e $M_v(t) = 1$ for all  $v,t$). We plot the same for  $d=1$, $\Bar{M}=2$ (i.e.,  $M_v(t) \leq 3$ for all  $v,t$) in Fig.~\ref{fig:d1m2} and for $d=2$, $\Bar{M}=2$ in Fig.~\ref{fig:d2m2}. 

Note that from Fig.~\ref{fig:d1m0} to Fig.~\ref{fig:d1m2}, $d$ was fixed ($d=1$) and $M$ increases from $0$ to $2$, where as from Fig.~\ref{fig:d1m2} to Fig.~\ref{fig:d2m2}, $\Bar{M}$ remains the same ($\Bar{M}=2$) and $d$ increases from $1$ to $2$. Interestingly, the number of samples required for almost sure recovery changes from around $1000$ to $9000$ from Fig.~\ref{fig:d1m0} to Fig.~\ref{fig:d1m2}. On the other hand, the number of required samples changes from $9000$ to $840000$ from Fig.~\ref{fig:d1m2} to Fig.~\ref{fig:d2m2}. In the second case, the increase in the number of required samples is substantially more. 

 In these simulations, since the degree of all the nodes in all the graphs is $1$, $|P|_{max}$ is actually a constant independent of $\Bar{M}$. Hence, in these cases, the analytical bound in Theorem~\ref{thm:2} depends on $\Bar{M}$ as $\Mbar^{2d}\log(\Mbar)$ through $|\xi|$. On the other hand, the analytical bound has a term that is exponential in $d$. Thus, the dependence of the constant terms (which are not dependent on $|V|$) of the bound on the parameters $d$ and $\Bar{M}$ corresponds well with the observations from simulations.

Note that the bound in Theorem~\ref{thm:2} is applicable when $2(\Bar{\mu} + L) |\rho(\tilde{A})|^{\frac{1}{d}}< 1$. All the previous simulations are done for a setting where that assumption is true. In Fig~\ref{fig:d2m2outside}, we plot the probability of success versus number of samples for $d=2$, $\Bar{M}=2$ when $2(\Bar{\mu} + L) |\rho(\tilde{A})|^{\frac{1}{d}}>1$. We observe that Algorithm~\ref{alg_recgreedy} can still recover the true graphs perfectly but needs around three times more samples than Fig.~\ref{fig:d2m2}. 

For the above scenario, i.e., $2(\Bar{\mu} + L) |\rho(\tilde{A})|^{\frac{1}{d}}>1$, we also observe the dependence of the number of required samples on $d$ and $\Bar{M}$. These are plotted in Tables \ref{table:d} and \ref{table:m}. Note that here also we see that the number of samples required for perfect recovery grows much faster with $d$ than with $\Bar{M}$ as in the analytical bound and as observed in Figures~\ref{fig:d1m0}--\ref{fig:d2m2}. This seems to indicate that the $O(\log |V|)$ bound on sample complexity is likely to be true even when $2(\Bar{\mu} + L) |\rho(\tilde{A})|^{\frac{1}{d}}>1$. We leave this for future analytical investigations.

\begin{table}[!]
\begin{center}
\caption{Relationship of $d$ to the number of samples required for perfect recovery}
\label{table:d}
\begin{tabular}{ |c|c |c|c|c| }
\hline 
 & $\Mbar$=0,d=1 & $\Mbar$=0, d=2 & $\Mbar$=0, d=3 & $\Mbar$=0, d=4 \Tstrut \\ 
 \hline 
  No. of samples & 3000 & 17000 & 48000 & 70000  \\
 \hline
\end{tabular}
\end{center}
\end{table}

\begin{table}[!]
\caption{Relationship of $\Mbar$ to the number of samples required for perfect recovery}
\label{table:m}
\begin{center}
\begin{tabular}{ |c|c |c|c|c| }
\hline  
 & d=1,$\Mbar$=0  & d=1,$\Mbar$=1 & d=1,$\Mbar$=2 & d=1,$\Mbar$=3 \Tstrut\\ 
\hline 
No. of samples & 3000  & 4000 & 12000 & 51000  \\
 \hline
\end{tabular}
\end{center}
\end{table}

\section{Conclusion}
\label{sec:conclusion}

We considered the problem of learning the influence or causal graph of a high-dimensional Markov process with memory. This problem is vital in many real-world applications, including social networks and nervous systems. We adapted the RecGreedy algorithm for learning i.i.d graphical models from \cite{Ray} by introducing directed conditional entropy as the metric for learning the neighborhood greedily. We proved a logarithmic sample complexity result by deriving a bound on the second eigenvalue of the Markov process with memory, which involved obtaining and comparing lower and upper bounds on the distance from the stationary distribution as a function of time. Unlike RecGreedy  \cite{Ray}, where the non-degeneracy of conditional entropy for the i.i.d graphical model is an unavoidable assumption, we presented a proof of non-degeneracy of the conditional directed entropy in our Markov model.  Simulations show that the adapted RecGreedy algorithm has a high probability of success even in cases where the analytical bound is not applicable.

\end{document}